\documentclass[11pt]{article}
\title{On Differentially Private Online Predictions}
\author{
Haim Kaplan\thanks{Tel Aviv University and Google Research. {\tt haimk@tau.ac.il}. Partially supported by Israel Science Foundation (grant 1595/19),  and the Blavatnik Family Foundation.}
\and
Yishay Mansour\thanks{Tel Aviv University and Google research. {\tt mansour.yishay@gmail.com}. Work partially funded from the European Research Council (ERC) under the European Union’s Horizon 2020 research and innovation program (grant agreement No. 882396), by the Israel Science Foundation (grant number 993/17), Tel Aviv University Center for AI and Data Science (TAD), and the Yandex Initiative for Machine Learning at Tel Aviv University.}
\and
Shay Moran\thanks{Departments of Mathematics and Computer Science, Technion and Google Research. {\tt smoran@technion.ac.il} \emph{Shay Moran} is a Robert J.\ Shillman Fellow; he acknowledges support by ISF grant 1225/20, by BSF grant 2018385, by an Azrieli Faculty Fellowship, by Israel PBC-VATAT, by the Technion Center for Machine Learning and Intelligent Systems (MLIS), and by the European Union (ERC, GENERALIZATION, 101039692). Views and opinions expressed are however those of the author(s) only and do not necessarily reflect those of the European Union or the European Research Council Executive Agency. Neither the European Union nor the granting authority can be held responsible for them.}
\and
Kobbi Nissim\thanks{Department of Computer Science, Georgetown University. {\tt kobbi.nissim@georgetown.edu}. Work partially funded by NSF grant No.~2001041 and by a gift to Georgetown University.}
\and
Uri Stemmer\thanks{Tel Aviv University and Google research. {\tt u@uri.co.il}. Partially supported by the Israel Science Foundation (grant 1871/19) and by Len Blavatnik and the Blavatnik Family foundation.}
}
\date{February 27, 2023}

\usepackage{caption}

\usepackage{natbib}

\usepackage{algorithm}

\usepackage{scalerel}
\usepackage{mathrsfs}
\usepackage{enumitem}
\usepackage{bbm}
\usepackage{tikz}
\usepackage{float}
\usetikzlibrary{backgrounds}
\usepackage{color}
\usepackage{graphicx}
\usepackage{latexsym}
\usepackage{amsfonts}
\usepackage{fullpage}
\usepackage{amsmath, amssymb, amsthm}
\usepackage{multirow}
\usepackage{dsfont}
\usepackage{array}
\usepackage{hyperref}
\usepackage{adjustbox}

\newcommand{\smallminus}{\scalebox{0.6}[1.0]{-}}

\newcommand{\dbtilde}[1]{\tilde{\raisebox{0pt}[0.85\height]{$\tilde{#1}$}}}

\def\restrict#1{\raise-.5ex\hbox{\ensuremath|}_{#1}}

\DeclareSymbolFont{AMSb}{U}{msb}{m}{n}
\DeclareMathSymbol{\N}{\mathbin}{AMSb}{"4E}
\DeclareMathSymbol{\Z}{\mathbin}{AMSb}{"5A}
\DeclareMathSymbol{\R}{\mathbin}{AMSb}{"52}
\DeclareMathSymbol{\Q}{\mathbin}{AMSb}{"51}
\DeclareMathSymbol{\erert}{\mathbin}{AMSb}{"50}
\DeclareMathSymbol{\I}{\mathbin}{AMSb}{"49}
\DeclareMathSymbol{\C}{\mathbin}{AMSb}{"43}

\newcommand{\mynote}[2]{{\textcolor{#1}{ #2}}}

\definecolor{gray}{gray}{0.4}
\newcommand{\gray}[1]{\mynote{gray}{{\footnotesize #1}}}

\newcommand{\remove}[1]{}

\newtheorem{theorem}{Theorem}[section]

\newtheorem{lemma}[theorem]{Lemma}

\newtheorem{definition}[theorem]{Definition}

\newtheorem{remark}[theorem]{Remark}

\newtheorem{proposition}[theorem]{Proposition}

\newtheorem{corollary}[theorem]{Corollary}

\newtheorem{example}[theorem]{Example}

\newcommand{\1}{\mathbbm{1}}

\newcommand{\AAA}{\mathcal A}
\newcommand{\BBB}{\mathcal B}
\newcommand{\WWW}{\mathcal W}

\newcommand{\HHH}{\mathcal H}

\newcommand{\III}{\mathcal I}

\newcommand{\MMM}{\mathcal M}

\newcommand{\ZZZ}{\mathcal Z}

\newcommand{\XXX}{\mathcal X}
\newcommand{\YYY}{\mathcal Y}
\newcommand{\eps}{\varepsilon}

\newcommand{\Lap}{\operatorname{\rm Lap}}

\newcommand{\Ldim}{\operatorname{\rm Ldim}}

\def\E{\operatorname*{\mathbb{E}}}

\def\Q{\operatorname*{\mathbb{Q}}}

\def\Lap{\mathop{\rm{Lap}}\nolimits}
\def\OPT{\mathop{\rm{OPT}}\nolimits}

\makeatletter
\def\@opargbegintheorem#1#2#3{\trivlist
\item[\hskip\dimexpr\labelsep+0pt\relax{\bf #1\ #2}]({\bf #3}).\ \itshape}
\makeatother

\newenvironment{proofsketch}{%
  \proof}{\endproof}

\begin{document}

\maketitle

\begin{abstract}

In this work we introduce an interactive variant of joint differential privacy towards handling online processes in which existing privacy definitions seem too restrictive. We study basic properties of this definition and demonstrate that it satisfies (suitable variants) of group privacy, composition, and post processing.

We then study the cost of interactive joint privacy in the basic setting of online classification. We show that any (possibly non-private) learning rule can be {\em effectively} transformed to a private learning rule with only a polynomial overhead in the mistake bound. This demonstrates a stark difference with more restrictive notions of privacy such as the one studied by \cite{GolowichL21}, %
where only 
a double exponential overhead on the mistake bound is known (via an information theoretic upper bound).
\end{abstract}

\section{Introduction}

In this work we introduce a new variant of differential privacy (DP)~\citep{DMNS06}, suitable for interactive processes, and design new online learning algorithms that satisfy our definition. As a motivating story, consider a chatbot that continuously improves itself by learning from the conversations it conducts with users. As these conversations might contain sensitive information, we would like to provide privacy guarantees to the users, in the sense that the content of their conversations with the chatbot would not leak. This setting flashes out the following two requirements. 
\begin{enumerate}
    \item[(1)] Clearly, the answers given by the chatbot to user $u_i$ must depend on the queries made by user $u_i$. 
    For example, the chatbot should provide different answers when asked by user $u_i$ for the weather forecast in Antarctica, and when asked by $u_i$ for a pasta recipe.

    This is in contrast to the plain formulation of differential privacy, where it is required that {\em all} of the mechanism outputs would be (almost) independent of any single user input. Therefore, the privacy requirement we are aiming for is that the conversation of user $u_i$ will remain ``hidden'' from {\em other} users, and would not leak through the {\em other} users' interactions with the chatbot. Moreover, this should remain true even if a ``privacy attacker'' (aiming to extract information about the conversation user $u_i$ had) conducts {\em many} different conversations with the chatbot.
    
    \item[(2)] The interaction with the chatbot is, by design, {\em interactive} and {\em adaptive}, as it aims to conduct dialogues with the users. This allows the privacy attacker (mentioned above) to choose its queries to the  chatbot {\em adaptively}. Privacy, hence, needs to be preserved even in the presence of adaptive attackers. %
\end{enumerate}

While each of these two requirements was studied in isolation, to the best of our knowledge, they have not been unified into a combined privacy framework. Requirement (1) was formalized by \cite{KearnsPRRU15} as {\em joint differential privacy (JDP)}. It provides privacy against {\em non-adaptive} attackers.
Intuitively, in the chatbot example, JDP aims to hide the conversation of user $u_i$ from any privacy attacker that  {\em chooses in advance} all the queries it poses to the chatbot. This is unsatisfactory since the adaptive nature of this process  invites adaptive attackers.

Requirement (2) was studied in many different settings, but to the best of our knowledge, only w.r.t.\ the plain formulation of DP, where the (adaptive) privacy attacker sees {\em all} of the outputs of the mechanism. Works in this vein include \citep{DNRRV09,ChanSS10,PMW-HR10,DRV10,BunSU16,KaplanMS21,JainRSS22}. 
In the chatbot example, plain DP would require, in particular, that even the messages sent from the chatbot to user $u_i$ would reveal (almost) no information about $u_i$. In theory, this could be obtained by making sure that the {\em entire} chatbot model is computed in a privacy preserving manner, such that even its full description leaks almost no information about any single user. Then, when user $u_i$ comes, we can ``simply'' share the model with her, and let her query it locally on her device. But is likely unrealistic with large models involving hundreds of billions of parameters.

In this work we introduce {\em challenge differential privacy}, which can be viewed as an interactive variant of JDP, aimed at maintaining privacy against adaptive privacy attackers. Intuitively, in the chatbot example, our definition would guarantee that even an adaptive attacker that controls {\em all} of the users except for user $u_i$, learns (almost) no information about the conversation user $u_i$ had with the chatbot. We give the formal definition of challenge-DP in Section~\ref{sec:challengeDP}, after surveying the existing variants of differential privacy in Section~\ref{sec:prelims}. In addition, we show that challenge-DP is closed under post-processing, composition, and group-privacy (where the first two properties are immediate, and the third is more subtle).

\subsection{Private Online Classification}

We initiate the study of challenge differential privacy in the basic setting of online classification.
Let $\XXX$ be the domain, $\YYY$ be the label space, and $\ZZZ=\XXX\times \YYY$ be set of labeled examples.
An online learner is a (possibly randomized) mapping $\AAA: \ZZZ^\star \times \XXX \to \YYY$.
That is, it is a mapping that maps a finite sequence $S\in \ZZZ^\star$ (the past examples),  
and an unlabeled example $x$ (the current query point) to a label $y$, which is denoted by $y=\AAA(x ; S)$.

Let $\mathcal{H}\subseteq \mathcal{Y}^\mathcal{X}$ be a hypothesis class. 
    A sequence $S\in \ZZZ^\star$ is said to be realizable by $\HHH$ if there exists $h\in\HHH$ such that $h(x_i)=y_i$ for every $(x_i,y_i)\in S$.
    For a sequence $S=\{(x_t,y_t)\}_{t=1}^T\in \ZZZ^\star$ we write $\MMM(\AAA; S)$ for the random variable denoting the number of mistakes $\AAA$ makes during the execution on $S$. That is
\[\MMM\bigl(\AAA;S\bigr) = \sum_{t=1}^T 1\{\hat y_t \neq y_t\},\] %
where $\hat y_t = \AAA(x_t;S_{<t})$ is the (randomized) prediction of $\AAA$ on $x_t$.
\begin{definition}[Online Learnability: Realizable Case]\label{def:onlineLearnability}
We say that a hypothesis class $\HHH$ is online learnable 
if there exists a learning rule $\AAA$ such that $\E\left[\MMM\bigl(\AAA;S\bigr)\right] =o(T)$ for every sequence $S$ which is realizable by $\HHH$.
\end{definition}

\begin{remark}
Notice that Definition~\ref{def:onlineLearnability} corresponds to an {\em oblivious} adversary, as it quantifies over the input sequence in advance. This should not be confused with the adversaries considered in the context of \emph{privacy} which are always adaptive in this work.
In the non-private setting, focusing on oblivious adversaries does not affect generality in terms of utility. This is less clear when privacy constraints are involved.\footnote{In particular, \cite{GolowichL21} studied both oblivious and adaptive adversaries, and obtained very different results in these two cases.} We emphasize that our results (our mistake bounds) continue to hold even when the realizable sequence is chosen by an adaptive (stateful) adversary, that at every point in time chooses the next input to the algorithm based on all of the previous outputs of the algorithm.
\end{remark}

A classical result due to \cite{Littlestone88online} characterizes online learnability (without privacy constraints) 
in terms of the Littlestone dimension. The latter is a combinatorial parameter of $\HHH$ which was named after Littlestone by~\cite{Ben-DavidPS09}. 

In particular, Littlestone's characterization implies the following dichotomy: if $\HHH$ has finite Littlestone dimension $d$ then there exists a (deterministic) learning rule which makes at most $d$ mistakes on every realizable input sequence. In the complementing case, when the Littlestone dimension of $\HHH$ is infinite, for every learning rule $\AAA$ and every $T\in\mathbb{N}$ there exists a realizable sequence $S$ of length $T$ such that $\E\left[\MMM\bigl(\AAA;S\bigr)\right] \geq T/2$. In other words, as a function of $T$, the optimal mistake bound is either uniformly bounded by the Littlestone dimension, or it is $\geq T/2$. Because of this dichotomy, in some places online learnability is defined with respect to a uniform bound on the number of mistakes (and not just a sublinear one as in the above definition). In this work we follow the more general definition. 

\medskip

We investigate the following questions:
\begin{center}
\emph{Can every online learnable class be learned by an algorithm which satisfies challenge differential privacy?
What is the optimal mistake bound attainable by private learners?
}
\end{center}
Our main result in this part provides an affirmative answer to the first question.
We show that for any class $\HHH$ with Littlestone dimension $d$ there exists an $(\eps,\delta)$-challenge-DP learning
rule which makes at most 
\[\tilde{O}\left( \frac{d^2}{\eps^2}  \log^2\left(\frac{1}{\delta}\right) \log^2\left(\frac{T}{\beta}\right) \right)\]
mistakes, with probability $1-\beta$, on every realizable sequence of length $T$.
\emph{Remarkably, our proof provides an efficient transformation taking a non-private learner to a private one:}
that is, given a black box access to a learning rule $\AAA$ which makes at most $M$ mistakes in the realizable case,
we efficiently construct an $(\eps,\delta)$-challenge-DP learning rule $\AAA'$ which makes at most 
$\tilde{O}\left( \frac{M^2}{\eps^2}  \log^2\left(\frac{1}{\delta}\right) \log^2\left(\frac{T}{\beta}\right) \right)$ mistakes.

\subsubsection{Construction overview}

We now give a simplified overview of our construction, called \texttt{POP}, which transforms a non-private online learning algorithm into a private one (while maintaining computational efficiency). 
Let $\AAA$ be a non-private algorithm, guaranteed to make at most $d$ mistakes in the realizable setting. 
We maintain $k$ copies of $\AAA$. Informally, in every round $i\in[T]$ we do the following: 
\begin{enumerate}[itemsep=-2pt]
    \item Obtain an input point $x_i$.
    \item\label{step:agg} Give $x_i$ to each of the $k$ copies of $\AAA$ to obtain predicted labels $\hat{y}_{i,1},\dots,\hat{y}_{i,k}$.
    \item Output a ``privacy preserving'' aggregation $\hat{y}_i$ of  $\left\{\hat{y}_{i,1},\dots,\hat{y}_{i,k}\right\}$, which is some variant of noisy majority. This step will only satisfy our notion of challenge-DP.
\item Obtain the ``true'' label $y_i$.
\item Let $\ell\in[k]$ be chosen at random.
\item Rewind all of the copies of algorithm $\AAA$ except for the $\ell$th copy, so that they ``forget'' ever seeing $x_i$.
\item Give the true label $y_i$ to the $\ell$th copy of $\AAA$.
\end{enumerate}

As we aggregate the predictions given by the copies of $\AAA$ using (noisy) majority, we know that if the algorithm errs than at least a constant fraction of the copies of $\AAA$ err. As we feed the true label $y_i$ to a random copy, with constant probability, the copy which we do not rewind incurs a mistake at this moment. That is, whenever we make a mistake then with constant probability one of the copies we maintain incurs a mistake. This can happen at most $\approx k\cdot d$ times, since we have $k$ copies and each of them makes at most $d$ mistakes. This allows us to bound the number of mistakes made by our algorithm (w.h.p.). The privacy analysis is more involved. Intuitively, by rewinding all of the copies of $\AAA$ (except one) in every round, we make sure that a single user can affect the inner state of at most one of the copies. This allows us to efficiently aggregate the predictions given by the copies in a privacy preserving manner. The subtle point is that the prediction we release in time $i$ {\em does} require querying {\em all} the experts on the current example $x_i$ (before rewinding them). Nevertheless, we show that this algorithm is private.

\subsubsection[Comparison with Golowich and Livni (2021)]{Comparison with \cite{GolowichL21}}
The closest prior work to this manuscript is by Golowich and Livni who also studied the problem of private online classification, but under a more restrictive notion of privacy than challenge-DP. In particular their definition requires that the sequence of \underline{predictors} which the learner uses to predict in each round does not compromise privacy. In other words, it is as if at each round the learner publishes the entire truth-table of its predictor, rather than just its current prediction. This might be too prohibitive in certain applications such as the chatbot example illustrated above. Golowich and Livni show that even with respect to their more restrictive notion of privacy it is possible to online learn every Littlestone class.
However, their mistake bound is doubly exponential in the Littlestone dimension (whereas ours is quadratic), and their construction requires more elaborate access to the non-private learner. In particular, it is not clear whether their construction can be implemented efficiently.

\subsection{Additional Related Work}

Several works studied the related problem of {\em private learning from expert advice} \citep{DworkNPR10,JainKT12,ThakurtaS13,DworkR14,JainThakurta14,Agarwal17a,AsiFKT23}. These works study a variant of the experts problem in which the learning algorithm has access to $k$ {\em experts}; on every time step the learning algorithm chooses one of the experts to follow, and then observes the {\em loss} of each expert. The goal of the learning algorithm is that its accumulated loss will be competitive with the loss of the {\em best expert in hindsight}. In this setting %
the private data is the sequence of losses observed throughout the execution, and the privacy requirement is that the sequence of experts chosen by the algorithm should not compromise the privacy of the sequence of losses.\footnote{\cite{AsiFKT23} study a more general framework of adaptive privacy in which the private data is an auxiliary sequence $(z_1,\ldots, z_T)$.  %
During the interaction with the learner, these $z_t$'s are used (possibly in an adaptive way) to choose the sequence of loss functions.} 
When applying these results to our context, the set of experts is the set of hypotheses in the class $\mathcal{H}$, which means that the outcome of the learner (on every time step) is a complete model (i.e., a hypothesis). That is, in our context, applying prior works on private prediction from expert advice would result in a privacy definition similar to that of \cite{GolowichL21} that accounts (in the privacy analysis) for releasing complete models, rather than just the predictions, which is significantly more restrictive.

There were a few works that studied private learning in online settings under the constraint of JDP. For example, \cite{ShariffS18} studied the stochastic contextual linear bandits problem under JDP. Here, in every round $t$ the learner 
receives a {\em context} $c_t$, then it selects an {\em action} $a_t$ (from a fixed set of actions), and finaly it receives a reward $y_t$ which depends on $(c_t,a_t)$ in a linear way. The learner's objective is to maximize cumulative reward. The (non-adaptive) definition of JDP means that action $a_t$ is revealed only to user $u_t$. Furthermore, it guarantees that the inputs of user $u_t$ (specifically the context $c_t$ and the reward $y_t$) do not leak to the other users via the actions they are given, provided that all these other users {\em fix their data in advance}. This non-adaptive privacy notion fits the stochastic setting of \cite{ShariffS18}, but (we believe) is less suited for adversarial processes like the ones we consider in this work. We also note that the algorithm of \cite{ShariffS18} in fact satisfies the more restrictive privacy definition which applies to the sequence of predictors (rather than the sequence of predictions), similarly to the that of \cite{GolowichL21}.

A parallel (unpublished) work by \cite{OtherPaper} studied a related setting, which can be viewed as an ``evolving'' variant of the private PAC learning model. They also use an adaptive variant of JDP, similar to our notion of privacy, which is tailored to their stochastic setting.

\section{Preliminaries}\label{sec:prelims}

\paragraph{Notation.}
Two datasets $S$ and $S'$ are called {\em neighboring} if one is obtained from the other by adding or deleting one element, e.g., $S'=S\cup\{x'\}$.
For two random variables $Y,Z$ we write $X\approx_{(\eps,\delta)}Y$ to mean that for every event $F$ it holds that 
$\Pr[X\in F] \leq e^{\eps}\cdot\Pr[Y\in F]+\delta$, and $\Pr[Y\in F]\leq e^{\eps}\cdot\Pr[X\in F]+\delta$. Throughout the paper we assume that the privacy parameter $\eps$ satisfies $\eps=O(1)$, but our analyses trivially extend to larger values of epsilon.

\medskip The standard definition of differential privacy is,

\begin{definition}[\citep{DMNS06}]\label{def:DP}
Let $\MMM$ be a randomized algorithm that operates on datasets.
Algorithm $\MMM$ is $(\eps,\delta)$-{\em differentially private (DP)} if for any two neighboring datasets $S,S'$ 
we have $\MMM(S)\approx_{(\eps,\delta)}\MMM(S')$.
\end{definition}

\paragraph{The Laplace mechanism.} The most basic constructions of differentially private algorithms are
via the Laplace mechanism as follows.

\begin{definition}
A random variable has probability distribution $\Lap(\gamma)$
if its probability density function is $f(x)=\frac{1}{2\gamma}\exp(-|x|/\gamma)$, where $x\in\R$.
\end{definition}

\begin{definition}[Sensitivity]
A function $f$ that maps datasets to the reals has {\em sensitivity $\Delta$} if for every two neighboring datasets $S$ and $S'$ it holds that $|f(S)-f(S')|\leq \Delta$.
\end{definition}

\begin{theorem}[The Laplace Mechanism \citep{DMNS06}]\label{thm:lap}
Let $f$ be a function that maps datasets to the reals with sensitivity $\Delta$. The mechanism $\AAA$ that on input $S$ adds noise with distribution $\Lap(\frac{\Delta}{\eps})$ to the output of $f(S)$ preserves $(\eps,0)$-differential privacy.
\end{theorem}

\paragraph{Joint differential privacy.} The standard definition of differential privacy (Definition~\ref{def:DP}) captures a setting in which
the entire output of the computation  may be publicly released without compromising privacy. While this is a very desirable requirement, it is sometimes too restrictive. Indeed, \cite{KearnsPRRU15} considered a relaxed setting in which we aim to analyze a dataset $S=(x_1,\dots,x_n)$, where every $x_i$ represents the information of user $i$, and to obtain a vector of  outcomes $(y_1,\dots,y_n)$. This vector, however, is not made public. Instead, every user $i$ only receives its ``corresponding outcome'' $y_i$. This setting potentially allows the outcome $y_i$ to strongly depend on the the input $x_i$, without compromising the privacy of the $i$th user from the view point of the other users.

\begin{definition}[\citep{KearnsPRRU15}]
Let $\MMM:X^n\rightarrow Y^n$ be a randomized algorithm that takes a dataset $S\in X^n$ and outputs a vector $\vec{y}\in Y^n$. Algorithm $\MMM$ satisfies $(\eps,\delta)$-joint differential privacy (JDP) if for every $i\in[n]$ and every two datasets $S,S'\in X^n$ differing only on their $i$th point it holds that 
$\MMM(S)_{-i}\approx_{(\eps,\delta)}\MMM(S')_{-i}$. Here $\MMM(S)_{-i}$ denotes the (random) vector of length $n-1$ obtained by running $(y_1,\dots,y_n)\leftarrow \MMM(S)$ and returning $(y_1,\dots,y_{i-1},y_{i+1},\dots,y_n)$.
\end{definition}

In words, consider an algorithm $\MMM$ that operates on the data of $n$ individuals and outputs $n$ outcomes $y_1,\dots,y_n$. This algorithm is JDP if changing only the $i$th input point $x_i$ has almost no affect on the outcome distribution of the {\em other} outputs (but the outcome distribution of $y_i$ is allowed to strongly depend on $x_i$).
\cite{KearnsPRRU15} showed that this setting fits a wide range of problems in economic environments. 

\begin{example}[\citep{nahmias2019privacy}] 
Suppose that a city water corporation is interested in promoting water conservation. To do so, the corporation decided to send each household a customized report indicating whether  their water consumption is above or below the median consumption in the neighborhood. Of course, this must be done in a way that protects the privacy of the neighbors. One way to tackle this would be to compute a privacy preserving estimation $z$ for the median consumption (satisfying Definition~\ref{def:DP}). Then, in each report, we could safely indicate whether the household's water consumption is bigger or smaller than $z$. While this solution is natural and intuitive, it turns out to be sub-optimal: We can obtain better utility by designing a JDP algorithm that directly computes a different outcome for each user (``above'' or ``below''), which is what we really aimed for, without going through a private median computation.
\end{example}

\paragraph{Algorithm \texttt{AboveThreshold}.}
Consider a large number of low sensitivity functions $f_1,f_2,\dots,f_T$ which are given (one by one) to a data curator (holding a dataset $S$). Algorithm \texttt{AboveThreshold}   allows for privately identifying the queries $f_i$ whose value $f_i(S)$ is (roughly) greater than some threshold $t$.

\begin{algorithm}[H]
\caption{\bf \texttt{AboveThreshold} \citep{DNRRV09,PMW-HR10}}\label{alg:AboveThreshold}
{\bf Input:} Dataset $S\in X^*$, privacy parameters $\eps,\delta$, threshold $t$, number of positive reports $r$, and an adaptively chosen stream of queries $f_i:X^*\rightarrow\R$ with sensitivity $\Delta$

\begin{enumerate}[topsep=-1pt,rightmargin=5pt,itemsep=-1pt]%
\item Denote $\gamma=O\left( \frac{\Delta}{\eps}\sqrt{r}\ln(\frac{r}{\delta}) \right)$
\item In each round $i$, when receiving a query $f_i\in Q$, do the following:
\begin{enumerate}[topsep=-3pt,rightmargin=5pt]%
\item Let $\hat{f_i}\leftarrow f_i(S)+\Lap(\gamma)$
\item If $\hat{f_i}\geq t$, then let $\sigma_i=1$ and otherwise let $\sigma_i=0$
\item Output $\sigma_i$
\item If $\sum_{j=1}^i \sigma_i\geq r$ then HALT
\end{enumerate}
\end{enumerate}
\end{algorithm}

Even though the number of possible rounds is unbounded, algorithm \texttt{AboveThreshold} preserves differential privacy. Note, however, that \texttt{AboveThreshold} is an {\em interactive} mechanism, while the standard definition of differential privacy (Definition~\ref{def:DP}) is stated for {\em non-interactive} mechanisms, that process their input dataset, release an output, and halt. The adaptation of DP to such interactive settings is done via a {\em game} between the (interactive) mechanism and an {\em adversary} that specifies the inputs to the mechanism and observes its outputs. Intuitively, the privacy requirement is that the view of the adversary at the end of the execution should be differentially private w.r.t.\ the inputs given to the mechanism. Formally,

\begin{definition}[DP under adaptive queries \citep{DMNS06,BunSU16}]\label{def:dpInteractiveQueries}
Let $\MMM$ be a mechanism that takes an input dataset and answers a sequence of adaptively chosen queries (specified by an adversary $\BBB$ and chosen from some family $Q$ of possible queries).
Mechanism $\MMM$ is $(\eps,\delta)$-differentially private if for every adversary $\BBB$ we have that $\texttt{AdaptiveQuery}_{\MMM,\BBB,Q}$ (defined below) is $(\eps,\delta)$-differentially private (w.r.t.\ its input bit $b$).
\end{definition}

\begin{algorithm*}[ht!]
\caption{\bf $\boldsymbol{\texttt{AdaptiveQuery}_{\MMM,\BBB,Q}}$ \citep{BunSU16}}\label{alg:adaptivealg}
{\bf Input:} A bit $b\in\{0,1\}$. (The bit $b$ is unknown to $\MMM$ and $\BBB$.)
\begin{enumerate}[topsep=-1pt,rightmargin=5pt,itemsep=-1pt]

\item The adversary $\BBB$ chooses two neighboring datasets $S_0$ and $S_1$.

\item The dataset $S_b$ is given to the mechanism $\MMM$.

\item For $i = 1,2,\dots$

\begin{enumerate}[topsep=-3pt,rightmargin=5pt]%
\item The adversary $\BBB$ chooses a query $q_i \in Q$.
\item The mechanism $\MMM$ is given $q_i$ and returns $a_i$.
\item $a_i$ is given to $\BBB$.
\end{enumerate}

\item When $\MMM$ or $\BBB$ halts, output $\BBB$'s view of the interaction, that is $(a_1,a_2,a_3,\cdots)$ and the internal randomness of $\BBB$.
\vspace{5px}
\end{enumerate}
\end{algorithm*}

\begin{theorem}[\citep{DNRRV09,PMW-HR10,KaplanMS21}]
Algorithm \texttt{AboveThreshold} is $(\eps,\delta)$-differentially private.
\end{theorem}

\paragraph{A private counter.}
In the setting of algorithm \texttt{AboveThreshold}, the dataset is fixed in the beginning of the execution, and the queries arrive sequentially one by one. \cite{DworkNPR10} and \cite{ChanSS10} considered a different setting, in which the {\em data} arrives sequentially. In particular, they considered the {\em counter} problem where in every time step $i\in[T]$ we obtain an input bit $x_i\in\{0,1\}$ (representing the data of user $i$) and must  immediately respond with an approximation for the current sum of the bits. That is, at time $i$ we wish to release an approximation for $x_1+x_2+\dots+x_i$.  

Similarly to our previous discussion, this is an {\em interactive} setting, and privacy is defined via a {\em game} between a mechanism $\MMM$ and an adversary $\BBB$ that adaptively determines the inputs for the mechanism.

\begin{definition}[DP under adaptive inputs \citep{DMNS06,DworkNPR10,ChanSS10,KaplanMS21,JainRSS22}]\label{def:dpInteractiveInputs}
Let $\MMM$ be a mechanism that in every round $i$ obtains an input point $x_i$ (representing the information of user $i$) and outputs a response $a_i$. 
Mechanism $\MMM$ is $(\eps,\delta)$-differentially private if for every adversary $\BBB$ we have that $\texttt{AdaptiveInput}_{\MMM,\BBB}$ (defined below) is $(\eps,\delta)$-differentially private (w.r.t.\ its input bit $b$).
\end{definition}

\begin{algorithm*}[ht!]
\caption{\bf $\boldsymbol{\texttt{AdaptiveInput}_{\MMM,\BBB}}$ \citep{JainRSS22}}
{\bf Input:} A bit $b\in\{0,1\}$. (The bit $b$ is unknown to $\MMM$ and $\BBB$.)
\begin{enumerate}[topsep=-1pt,rightmargin=5pt,itemsep=-1pt]

\item For $i = 1,2,\dots$

\begin{enumerate}[topsep=-3pt,rightmargin=5pt]%

\item The adversary $\BBB$ outputs a bit $c_i\in\{0,1\}$, under the restriction that $\sum_{j=1}^i c_j\leq 1$.\\
{\small \gray{\% The round $i$ in which $c_i=1$ is called the {\em challenge} round. Note that there could be at most one challenge round throughout the game.}}

\item The adversary $\BBB$ chooses two input points $x_{i,0}$ and $x_{i,1}$, under the restriction that if $c_i=0$ then $x_{i,0}=x_{i,1}$.

\item Algorithm $\MMM$ obtains $x_{i,b}$ and outputs $a_i$.

\item $a_i$ is given to $\BBB$.
\end{enumerate}

\item When $\MMM$ or $\BBB$ halts, output $\BBB$'s view of the interaction, that is $(a_1,a_2,a_3,\cdots)$ and the internal randomness of $\BBB$.
\vspace{5px}
\end{enumerate}
\end{algorithm*}

\begin{theorem}[Private counter \citep{DworkNPR10,ChanSS10,JainRSS22}]\label{thm:counter}
There exists a mechanism $\MMM$ that in each round $i\in[T]$ obtains an input bit $x_i\in\{0,1\}$ and outputs a response $a_i\in\N$ with the following properties:
\begin{enumerate}
    \item $\MMM$ is $(\eps,0)$-differentially private (as in Definition~\ref{def:dpInteractiveInputs}).
    \item Let $s$ denote the random coins of $\MMM$. Then there exists an event $E$ such that: (1) $\Pr[s\in E]\geq 1-\beta$, and (2) Conditioned on every $s\in E$, for {\em every} input sequence $(x_1,\dots,x_T)$, the answers $(a_1,\dots,a_T)$ satisfy $$\left| a_i-\sum_{j=1}^i x_i\right|\leq O\left(\frac{1}{\eps}\log(T)\log\left(\frac{T}{\beta}\right)\right).$$
\end{enumerate}
\end{theorem}

\section{Challenge Differential Privacy}\label{sec:challengeDP}

We now introduce the privacy definition we consider in this work is. Intuitively, the requirement is that even an adaptive adversary controlling all of the users except Alice, cannot learn much information about the interaction Alice had with the algorithm.  

\begin{definition}\label{def:CDP}
Consider an algorithm $\MMM$ that, in each round $i\in[T]$ obtains an input point $x_i$, outputs a ``predicted'' label $\hat{y}_i$, and obtains a ``true'' label $y_i$. We say that algorithm $\MMM$ is {\em $(\eps,\delta)$-challenge differentially private} if for any adversary $\BBB$ we have that  $\texttt{OnlineGame}_{\MMM,\BBB,T}$, defined below, is $(\eps,\delta)$-differentially private (w.r.t.\ its input bit $b$).
\end{definition}

\begin{remark}
For readability, we have simplified Definition~\ref{def:CDP} and tailored it to the setting of online learning. Our algorithms satisfy a stronger variant of the definition, in which the adversary may adaptively choose the ``true'' labels $y_i$ also based on the ``predicted'' labels $\hat{y}_i$. See Appendix~\ref{sec:generalDef} for the generalized definition.
\end{remark}

\begin{algorithm*}[ht!]
\caption{\bf $\boldsymbol{\texttt{OnlineGame}_{\MMM,\BBB,T,g}}$}\label{alg:Game}

{\bf Setting:} $T\in\N$ denotes the number of rounds and $g\in\N$ is a ``group privacy'' parameter. If not  explicitly stated  we assume that $g=1$. $\MMM$ is an online algorithm and $\BBB$ is an adversary that determines the inputs adaptively.

{\bf Input of the game:} A bit $b\in\{0,1\}$. (The bit $b$ is unknown to $\MMM$ and $\BBB$.)
\begin{enumerate}[topsep=-1pt,rightmargin=5pt,itemsep=-1pt]

\item For $i = 1,2,\dots,T$

\begin{enumerate}[topsep=-3pt,rightmargin=5pt]%

\item The adversary $\BBB$ outputs a bit $c_i\in\{0,1\}$, under the restriction that $\sum_{j=1}^i c_j\leq g$.\\
{\small \gray{\% We interpret rounds $i$ in which $c_i=1$ as {\em challenge} rounds. Note that there could be at most $g$ challenge rounds throughout the game.}}

\item The adversary $\BBB$ chooses two labeled inputs $(x_{i,0},y_{i,0})$ and $(x_{i,1},y_{i,1})$, under the restriction that if $c_i=0$ then $(x_{i,0},y_{i,0})=(x_{i,1},y_{i,1})$.

\item Algorithm $\MMM$ obtains $x_{i,b}$, then outputs $\hat{y}_i$, and then obtains $y_{i,b}$.

\item If $c_i=0$ then set $\tilde{y}_i=\hat{y}_i$. Otherwise set $\tilde{y}_i=\bot$.

\item The adversary $\BBB$ obtains $\tilde{y}_i$.\\
{\small \gray{\% Note that the adversary $\BBB$ does not get to see the outputs of $\MMM$ in challenge rounds.}}
\end{enumerate}

\item Output $\BBB$'s view of the game, that is $\tilde{y}_1,\dots,\tilde{y}_T$ and the internal randomness of $\BBB$.\\
{\small \gray{\% Note that from this we can reconstruct all the input points  $x_{i,0},x_{i,1}$ specified by $\BBB$ throughout the game.}}
\vspace{5px}
\end{enumerate}

\end{algorithm*}

\paragraph{Composition and post-processing.}
Composition and post-processing for challenge-DP follows immediately from their analogues for  (standard) DP. Formally, composition is defined via the following game, called $\texttt{CompositionGame}$, in which a ``meta adversary'' $\BBB^*$ is trying  to guess an unknown bit $b\in\{0,1\}$.
The meta adversary $\BBB^*$ is allowed to (adaptively) invoke $k$  executions of the game specified in Algorithm~\ref{alg:Game}, where all of these $k$ executions are done with the same (unknown) bit $b$. See Algorithm~\ref{alg:CompositionGame}. %
The following theorem follows immediately from standard composition theorems for differential privacy \citep{DRV10}.

\begin{algorithm*}[ht!]
\caption{\bf $\boldsymbol{\texttt{CompositionGame}_{\BBB^*,m,\eps,\delta}}$}\label{alg:CompositionGame}

{\bf Input of the game:} A bit $b\in\{0,1\}$. (The bit $b$ is unknown to $\BBB^*$.)
\begin{enumerate}[topsep=-1pt,rightmargin=5pt,itemsep=-1pt]

\item For $\ell = 1,2,\dots,m$

\begin{enumerate}[topsep=-3pt,rightmargin=5pt]%

\item The adversary $\BBB^*$ outputs an $(\eps,\delta)$-challenge-DP algorithm $\MMM_{\ell}$, an adversary $\BBB_{\ell}$, and an integer $T_{\ell}$.

\item The adversary $\BBB^*$ obtains the outcome of $\texttt{OnlineGame}_{\MMM_{\ell},\BBB_{\ell},T_{\ell}}(b)$.

\end{enumerate}

\item Output $\BBB^*$'s view of the game (its internal randomness and all of the outcomes of $\texttt{OnlineGame}$ it obtained throughout the execution).
\vspace{5px}
\end{enumerate}

\end{algorithm*}

\begin{theorem}[special case of \citep{DRV10}]\label{thm:composition}
For every $\BBB^*$, every $m\in\N$ and every $\eps,\delta,\delta'\geq0$ it holds that $\texttt{CompositionGame}_{\BBB^*,m,\eps,\delta}$ 
is $(\eps', m\delta+\delta')$-differentially private (w.r.t.\  the input bit $b$) for
$$
\eps'=\sqrt{2m \ln(1/\delta')}\eps + m\eps(e^{\eps}-1).
$$
\end{theorem}

\paragraph{Group privacy.} We show that challenge-DP is closed under group privacy. This is more subtle than the composition argument. In fact, we first need to {\em define} what do we mean by ``group privacy'' in the context of challenge-DP. This is done using the parameter $g$ in algorithm \texttt{OnlineGame}.

\begin{theorem}\label{thm:groupprivacy}
Let $\MMM$ be an algorithm that in each round $i\in[T]$ obtains an input point $x_i$, outputs a ``predicted'' label $\hat{y}_i$, and obtains a ``true'' label $y_i$. If $\MMM$ is $(\eps,\delta)$-challenge-DP then for every $g\in\N$ and every adversary $\BBB$ (posing at most $g$ challenges) we have that $\texttt{OnlineGame}_{\MMM,\BBB,T,g}$ is $(g\eps, g\cdot e^{\eps g}\cdot\delta)$-differentially private. 
\end{theorem}

\begin{proof}
Fix $g\in\N$ and fix an adversary $\BBB$ (that poses at most $g$ challenge rounds). We consider a sequence of games $\WWW_0,\WWW_1,\dots,\WWW_g$, where $\WWW_{\ell}$ is defined as follows.
\begin{enumerate}
    \item Initialize algorithm $\MMM$ and the adversary $\BBB$.
    \item For round $i=1,2,\dots,T$:
    \begin{enumerate}[leftmargin=23pt]
        \item Obtain a challenge indicator $c_i$ and two labeled inputs $(x_{i,0},y_{i,0})$ and $(x_{i,1},y_{i,1})$ from $\BBB$.
        \item If $\sum_{j=1}^i c_j > \ell$ then set $(w_i,z_i)=(x_{i,0},y_{i,0})$. Otherwise set $(w_i,z_i)=(x_{i,1},y_{i,1})$.
        \item Feed $w_i$ to algorithm $\MMM$, obtain an outcome $\hat{y}_i$, and feed it $z_i$.
        \item If $c_i=0$ then set $\tilde{y}_i=\hat{y}_i$. Otherwise set $\tilde{y}_i=\bot$.
        \item Give $\tilde{y}_i$ to $\BBB$.
    \end{enumerate}
    \item Output $\tilde{y}_1,\dots,\tilde{y}_T$ and the internal randomness of $\BBB$.
\end{enumerate}

That is, $\WWW_{\ell}$ simulates the online game between $\MMM$ and $\BBB$, where during the first $\ell$ challenge rounds algorithm $\MMM$ is given $(x_{i,1},y_{i,1})$, and in the rest of the challenge rounds algorithm $\MMM$ is given $(x_{i,0},y_{i,0})$. 
Note that 
$$\texttt{OnlineGame}_{\MMM,\BBB,T,g}(0)\equiv\WWW_0
\qquad\text{and}\qquad
\texttt{OnlineGame}_{\MMM,\BBB,T,g}(1)\equiv\WWW_g.
$$
We claim that for every $0<\ell\leq g$ it holds that $\WWW_{\ell-1} \approx_{(\eps,\delta)}\WWW_{\ell}$. To this end, fix $0<\ell\leq g$ and consider an adversary $\widehat{\BBB}$, that poses at most one challenge, defined as follows. 
Algorithm $\widehat{\BBB}$ runs $\BBB$ internally. In every round $i$, algorithm $\widehat{\BBB}$ obtains from $\BBB$ a challenge bit $c_i$ and two labeled inputs $(x_{i,0},y_{i,0})$ and $(x_{i,1},y_{i,1})$. As long as $\BBB$ did not pose its $\ell$th challenge, algorithm $\widehat{\BBB}$ outputs $(x_{i,1},y_{i,1}),(x_{i,1},y_{i,1})$. During the round $i$ in which $\BBB$ poses its $\ell$th challenge, algorithm $\BBB$ outputs $(x_{i,0},y_{i,0}),(x_{i,1},y_{i,1})$. This is the challenge round posed by algorithm $\widehat{\BBB}$. In every round $t$ afterwards, algorithm $\widehat{\BBB}$ outputs $(x_{i,0},y_{i,0}),(x_{i,0},y_{i,0})$. When algorithm $\widehat{\BBB}$ obtains an answer $\tilde{y}_i$ it sets $\dbtilde{y}_i =
\begin{cases}
\tilde{y}_i, \text{ if } c_i=0\\
\bot, \text{ if } c_i=1
\end{cases}
$ and gives $\dbtilde{y}_i$ to algorithm $\BBB$.

As $\widehat{\BBB}$ is an adversary that poses (at most) one challenge, by the privacy properties of $\MMM$ we know that $\texttt{OnlineGame}_{\MMM,\widehat{\BBB},T}$ is $(\eps,\delta)$-DP. 
Recall that the output of $\texttt{OnlineGame}_{\MMM,\widehat{\BBB},T}$ includes all of the randomness of $\widehat{\BBB}$, as well as the answers $\tilde{y}_t$ generated throughout the game. This includes the randomness of $\BBB$ (which $\widehat{\BBB}$ runs internally), and hence, determines also all of the $\dbtilde{y}_i$'s defined by $\widehat{\BBB}$ throughout the interaction. Let $P$ be a post-processing procedure that takes the output of $\texttt{OnlineGame}_{\MMM,\widehat{\BBB},T}$ and returns the randomness of $\BBB$ as well as $(\dbtilde{y}_1,\dots,\dbtilde{y}_T)$. By closure of DP to post-processing, we have that
$$
P(\texttt{OnlineGame}_{\MMM,\widehat{\BBB},T}(0))
\approx_{(\eps,\delta)}
P(\texttt{OnlineGame}_{\MMM,\widehat{\BBB},T}(1)).
$$
Now note that 
$$
P(\texttt{OnlineGame}_{\MMM,\widehat{\BBB},T}(0))
\equiv
\WWW_{\ell-1}
\qquad\text{and}\qquad
P(\texttt{OnlineGame}_{\MMM,\widehat{\BBB},T}(1))\equiv
\WWW_{\ell},
$$
and hence 
$\WWW_{\ell-1}\approx_{(\eps,\delta)}\WWW_{\ell}$. Overall we have that
$$
\texttt{OnlineGame}_{\AAA,\BBB,T,g}(0)
\equiv
\WWW_0
\approx_{(\eps,\delta)}
\WWW_1
\approx_{(\eps,\delta)}
\WWW_2
\approx_{(\eps,\delta)}
\dots
\approx_{(\eps,\delta)}
\WWW_g
\equiv
\texttt{OnlineGame}_{\AAA,\BBB,T,g}(1).
$$
This shows that  $\texttt{OnlineGame}_{\AAA,\BBB,T,g}$ is $(g\eps, g\cdot e^{\eps g}\cdot\delta)$-differentially private, thereby completing the proof.
\end{proof}

\section{Online Classification under Challenge Differential Privacy}

Towards presenting our private online learner, we introduce a variant of algorithm \texttt{AboveThreshold} with additional guarantees, which we call \texttt{ChallengeAT}. Recall that \texttt{AboveThreshold} ``hides'' arbitrary modifications to a single input point. Intuitively, the new variant we present aims to hide both an arbitrary modification to a single input point and an arbitrary modification to a single query throughout the execution. Consider algorithm \texttt{ChallengeAT}.

\begin{algorithm}
\caption{\bf \texttt{ChallengeAT}}
{\bf Input:} Dataset $S\in X^*$, privacy parameters $\eps,\delta$, threshold $t$, number of positive reports $r$, and an adaptively chosen stream of queries $f_i:X^*\rightarrow\R$ each with sensitivity $\Delta$

{\bf Tool used:} An $(\eps,0)$-DP algorithm, \texttt{PrivateCounter}, for counting bits under continual observation, guaranteeing error at most $\lambda$ with probability at least $1-\delta$
\begin{enumerate}[topsep=-1pt,rightmargin=5pt,itemsep=-1pt]%
\item Instantiate \texttt{PrivateCounter}
\item Denote $\gamma=O\left( \frac{\Delta}{\eps}\sqrt{r+\lambda}\ln(\frac{r+\lambda}{\delta}) \right)$
\item In each round $i$, when receiving a query $f_i$, do the following:
\begin{enumerate}[topsep=-3pt,rightmargin=5pt]%
\item Let $\hat{f_i}\leftarrow f_i(S)+\Lap(\gamma)$
\item If $\hat{f_i}\geq t$, then let $\sigma_i=1$ and otherwise let $\sigma_i=0$
\item Output $\sigma_i$
\item\label{step:FeedCounter} Feed $\sigma_i$ to \texttt{PrivateCounter} and let ${\rm count}_i$ denote its current output
\item If ${\rm count}_i\geq r$ then HALT
\end{enumerate}
\end{enumerate}
\end{algorithm}

\begin{remark}
When we apply \texttt{ChallengeAT}, it sets
$\lambda=O\left(\frac{1}{\eps}\log(T)\log\left(\frac{T}{\beta}\right)\right)$. Technically, for this it has to know $T$ and $\beta$. To simplify the description this is not explicit in our algorithms.    
\end{remark}

The utility guarantees of \texttt{ChallengeAT} are straightforward. The following theorem follows by bounding (w.h.p.) all the noises sampled throughout the execution (when instantiating \texttt{ChallengeAT} with the private counter from Theorem~\ref{thm:counter}).\footnote{The event $E$ occurs when all the Laplace noises of the counter and \texttt{ChallengeAT} are within a factor of $\log(T/\beta)$ of their expectation.}

\begin{theorem}\label{thm:ChallengeAT}
Let $s$ denote the random coins of \texttt{ChallengeAT}. Then there exists an event $E$ such that: (1) $\Pr[s\in E]\geq 1-\beta$, and (2) Conditioned on every $s\in E$, for {\em every} input dataset $S$ and {\em every} sequence of $T$ queries $(f_1,\dots,f_T)$ it holds that
\begin{enumerate}
    \item Algorithm \texttt{ChallengeAT} does not halt before the $r$th time in which it outputs $\sigma_i=1$.
    \item For every $i$ such that $\sigma_i=1$ it holds that $f_i(S)\geq t - O\left( \frac{\Delta}{\eps}\sqrt{r+\lambda}\ln(\frac{r+\lambda}{\delta})\log(\frac{T}{\beta}) \right)$
    \item For every $i$ such that $\sigma_i=0$ it holds that $f_i(S)\leq t + O\left( \frac{\Delta}{\eps}\sqrt{r+\lambda}\ln(\frac{r+\lambda}{\delta})\log(\frac{T}{\beta}) \right)$
\end{enumerate}
where $\lambda=O\left(\frac{1}{\eps}\log(T)\log\left(\frac{T}{\beta}\right)\right)$ is the error of the counter of Theorem~\ref{thm:counter}.
\end{theorem}

The privacy guarantees of \texttt{ChallengeAT} are defined via a game with an adversary $\BBB$ whose goal is to guess a secret bit $b$. At the beginning of the game, the adversary chooses two neighboring datasets $S_0,S_1$, and \texttt{ChallengeAT} is instantiated with $S_b$. Then throughout the game the adversary specifies queries $f_i$ and observes the output of \texttt{ChallengeAT} on these queries. At some special round $i^*$, chosen by the adversary, the adversary specifies {\em two} queries $f_{i^*}^0,f_{i^*}^1$, where only $f_{i^*}^b$ is fed into \texttt{ChallengeAT}. In round $i^*$ the adversary does not get to see the answer of \texttt{ChallengeAT} on $f_{i^*}^b$ (otherwise it could easily learn the bit $b$ since $f_{i^*}^0,f_{i^*}^1$ may be very different). The formal statement of this game is given in algorithm $\texttt{ChallengeAT{\smallminus}Game}_{\BBB}$.

\begin{algorithm*}[ht!]
\caption{\bf $\boldsymbol{\texttt{ChallengeAT{\smallminus}Game}_{\BBB}}$}

{\bf Setting:} $\BBB$ is an adversary that adaptively determines the inputs to \texttt{ChallengeAT}.

{\bf Input of the game:} A bit $b\in\{0,1\}$. (The bit $b$ is unknown to \texttt{ChallengeAT} and $\BBB$.)
\begin{enumerate}[topsep=-1pt,rightmargin=5pt,itemsep=-1pt]

\item The adversary $\BBB$ specifies two neighboring datasets $S_0,S_1\in X^*$.

\item Instantiate \texttt{ChallengeAT} with the dataset $S_b$ and parameters $\eps,\delta$, threshold $t$, and number of positive reports $r$.

\item For $i = 1,2,3,\dots$

\begin{enumerate}[topsep=-3pt,rightmargin=5pt]%

\item Get bit $c_i\in\{0,1\}$ from $\BBB$ subject to the restriction that $\sum_{j=1}^i c_j\leq 1$.\\
{\small \gray{\% When $c_i=1$ this is the \emph{Challange round}.}}

\item Get two queries $f^0_i:X^*\rightarrow\R$ and $f^1_i:X^*\rightarrow\R$ from $\BBB$, each with sensitivity $\Delta$, subject to  the restriction that if $c_i=0$ then $f^0_i\equiv f^1_i$.

\item Give the query $f^b_i$ to \texttt{ChallengeAT} and get back the bit $\sigma_i$.

\item If $c_i=0$ then set $\hat{y}_i=\sigma_i$. Otherwise set $\hat{y}_t=\bot$.

\item Give $\hat{y}_i$ to the adversary $\BBB$.
\end{enumerate}

\item Publish $\BBB$'s view of the game, that is $\hat{y}_1,\dots,\hat{y}_T$ and the internal randomness of $\BBB$.
\vspace{5px}
\end{enumerate}

\end{algorithm*}

\begin{theorem}
    For every adversary $\BBB$ it holds that $\texttt{ChallengeAT{\smallminus}Game}_{\BBB}$ is $\left(O(\eps),O(\delta)\right)$-DP w.r.t.\ the bit $b$ (the input of the game).
\end{theorem}

\begin{proof}
Fix an adversary $\BBB$. Let 
 $\texttt{CATG}$ denote the algorithm
 $\texttt{ChallengeAT{\smallminus}Game}_{\BBB}$ with this fixed 
$\BBB$.
Consider a variant of algorithm $\texttt{CATG}$, which we call $\texttt{CATG}{\smallminus}{\rm noCount}$ defined as follows. During the challenge round $i$, inside the call to \texttt{ChallengeAT}, instead of feeding $\sigma_i$ to the \texttt{PrivateCounter} we simply feed  it 0 (in Step~\ref{step:FeedCounter} of \texttt{ChallengeAT}). 

By the privacy properties of \texttt{PrivateCounter} (Theorem \ref{thm:counter}), for every $b\in\{0,1\}$ we have that
$$
\texttt{CATG}(b)
\approx_{(\eps,0)} \texttt{CATG}{\smallminus}{\rm noCount}(b),
$$
so it suffices to show that $\texttt{CATG}{\smallminus}{\rm noCount}$ is  DP (w.r.t.\ $b$). Now observe that the execution of \texttt{PrivateCounter} during the execution of $\texttt{CATG}{\smallminus}{\rm noCount}$ 
can be simulated from the view of the adversary $\BBB$ (the only
bit that \texttt{ChallengeAT} feeds the counter which is not in the view of the adversary is the one of the challange round which we replaced by zero in $\texttt{CATG}{\smallminus}{\rm noCount}$).
Hence, 
we can generate the view of $\BBB$ in algorithm $\texttt{CATG}$
by interacting with \texttt{AboveThreshold} instead of with \texttt{ChallengeAT}. 
This is captured by algorithm 
 $\texttt{CAT{\smallminus}G}{\smallminus}{\rm AboveThrehold}$.

\begin{algorithm*}[ht!]
\caption{\bf $\boldsymbol{\texttt{CAT{G}\smallminus}{\rm AboveThreshold}}$}

{\bf Setting:} $\BBB$ is an adversary that adaptively determines the inputs to \texttt{ChallengeAT}.

{\bf Input of the game:} A bit $b\in\{0,1\}$. (The bit $b$ is unknown to \texttt{ChallengeAT} and $\BBB$.)
\begin{enumerate}[topsep=-1pt,rightmargin=5pt,itemsep=-1pt]

\item The adversary $\BBB$ specifies two neighboring datasets $S_0,S_1\in X^*$.

\item Instantiate \texttt{PrivateCounter}

\item Instantiate \texttt{AboveThreshold} on the dataset $S_b$ with parameters $\eps,\delta,t,(r+\lambda)$.

\item For $i = 1,2,3,\dots$

\begin{enumerate}[topsep=-3pt,rightmargin=5pt]%

\item Get bit $c_i\in\{0,1\}$ from the adversary $\BBB$ subject to the restriction that $\sum_{j=1}^i c_j\leq 1$.

\item Get two queries $f^0_i:X^*\rightarrow\R$ and $f^1_i:X^*\rightarrow\R$, each with sensitivity $\Delta$ from $\BBB$, subject to the restriction that if $c_i=0$ then $f^0_i\equiv f^1_i$.

\item\label{step:feedAT} 
Give the query $f^b_i$ to
Algorithm \texttt{AboveThreshold} and get back a bit $\sigma_i$.

\item If $c_i=0$ then set $\hat{y}_i=\sigma_i$. Otherwise set $\hat{y}_t=\bot$.

\item Give $\hat{y}_i$ to the adversary $\BBB$.

\item If $c_i=0$ then feed $\sigma_i$ to \texttt{PrivateCounter}, and otherwise feed it $0$.

\item Let ${\rm count}_i$ denote the current output of \texttt{PrivateCounter}, and HALT if ${\rm count}_i\geq r$
\end{enumerate}

\item Publish $\BBB$'s view of the game, that is $\hat{y}_1,\dots,\hat{y}_T$ and the internal randomness of $\BBB$.
\vspace{5px}
\end{enumerate}

\end{algorithm*}

This algorithm is almost identical to $\texttt{CATG}{\smallminus}{\rm noCount}$, except for the fact that \texttt{AboveThreshold} might halt the execution itself (even without the halting condition on the outcome of \texttt{PrivateCounter}). However, by the utility guarantees of \texttt{PrivateCounter}, with probability at least $1-\delta$ it never errs by more than $\lambda$, in which case algorithm \texttt{AboveThreshold} never halts prematurely. Hence, for every bit $b\in\{0,1\}$ we have that
$$
\texttt{CATG}{\smallminus}{\rm AboveThrehold}(b)\approx_{(0,\delta)}\texttt{CATG}{\smallminus}{\rm noCount}(b).$$ 
So it suffices to show that $\texttt{CATG}{\smallminus}{\rm AboveThrehold}$ is DP (w.r.t.\ its input bit $b$).
This almost follows directly from the privacy guarantees of \texttt{AboveThreshold}, since $\texttt{CATG}{\smallminus}{\rm AboveThrehold}$ interacts only with this algorithm, except for the fact that during the challenge round $i$ the adversary $\BBB$ specifies two queries (and only one of them is fed into \texttt{AboveThreshold}). To bridge this gap, we consider one more (and final) modification to the algorithm, called $\widehat{\texttt{CATG}}{\smallminus}{\rm AboveThrehold}$. This algorithm  is identical to $\texttt{CATG}{\smallminus}{\rm AboveThrehold}$, except that in Step~\ref{step:feedAT} we do not  feed $f_i^b$ to \texttt{AboveThreshold} if $c_i=1$. That is, during the challenge round we do not interact with \texttt{AboveThreshold}.

Now, by the privacy properties of \texttt{AboveThreshold} we have that $\widehat{\texttt{CATG}}{\smallminus}{\rm AboveThrehold}$ is DP (w.r.t.\ its input bit $b$). Furthermore, when algorithm \texttt{AboveThreshold} does not halt prematurely, we have that $\widehat{\texttt{CATG}}{\smallminus}{\rm AboveThrehold}$ is identical to $\texttt{CATG}{\smallminus}{\rm AboveThrehold}$. Therefore, for every bit $b\in\{0,1\}$ we have
$$\texttt{CATG}{\smallminus}{\rm AboveThrehold}(b)\approx_{(0,\delta)}\widehat{\texttt{CATG}}{\smallminus}{\rm AboveThrehold}(b).$$ 

Overall we get that
\begin{align*}
\texttt{CATG(0)}&
\approx_{(\eps,0)} \texttt{CATG}{\smallminus}{\rm noCount}(0)\\
&\approx_{(0,\delta)}\texttt{CATG}{\smallminus}{\rm AboveThrehold}(0)\\
&\approx_{(0,\delta)}\widehat{\texttt{CATG}}{\smallminus}{\rm AboveThrehold}(0)\\
&\approx_{(\eps,\delta)}\widehat{\texttt{CATG}}{\smallminus}{\rm AboveThrehold}(1)\\
&\approx_{(0,\delta)}\texttt{CATG}{\smallminus}{\rm AboveThrehold}(1)\\
&\approx_{(0,\delta)} \texttt{CATG}{\smallminus}{\rm noCount}(1)\\
&\approx_{(\eps,0)}
\texttt{CATG(1)}
\end{align*}

\end{proof}

\subsection{Algorithm \texttt{POP}}

We are now ready to present our private online prediction algorithm. Consider algorithm \texttt{POP} (see Algorithm \ref{alg:Realizable}).

\begin{algorithm*}
\caption{\bf $\boldsymbol{\texttt{POP (Private Online Procedure)}}$}\label{alg:Realizable}

{\bf Setting:} $T\in\N$ denotes the number of rounds in the game. $\AAA$ is a non-private online-algorithm.

{\bf Parameters:} $k$ determines the number of copies of $\AAA$ we maintain. $r$  determines the number of positive reports we aim to receive from \texttt{ChallengeAT}.

\begin{enumerate}[topsep=-1pt,rightmargin=5pt,itemsep=-1pt]

\item Instantiate $k$ copies $\AAA_1,\dots,\AAA_k$ of algorithm $\AAA$

\item Instantiate algorithm \texttt{ChallengeAT} on an empty dataset with threshold $t=-k/4$, privacy parameters $\eps,\delta$, number of positive reports $r$, and sensitivity parameter $\Delta=1$.

\item For $i = 1,2,\dots,T$

\begin{enumerate}[topsep=-3pt,rightmargin=5pt]%

\item Obtain input $x_i$

\item Let $\AAA^{\rm temp}_1,\dots,\AAA^{\rm temp}_k$ be duplicated copies of $\AAA_1,\dots,\AAA_k$

\item Let $\ell_i\in[k]$ be chosen uniformly at random

\item Let $\hat{y}_{i,{\ell_i}}\leftarrow\AAA_{\ell_i}(x_i)$. For $j\in[k]\setminus\{\ell_i\}$ let $\hat{y}_{i,j}\leftarrow\AAA^{\rm temp}_j(x_i)$

\item\label{step:ft} Feed \texttt{ChallengeAT} the query  $f_i\equiv-\left|\frac{k}{2}-\sum_{j\in[k]}\hat{y}_{i,j}\right|$ and obtain an outcome $\sigma_i$. (If \texttt{ChallengeAT} halts then \texttt{POP} also halts.)\\
{\small \gray{\% Recall that $\sigma_i=1$ indicates that $-\left|\frac{k}{2}-\sum_{j\in[k]}\hat{y}_{i,j}\right|\gtrsim-\frac{k}{4}$, meaning that there is ``a lot'' of disagreement among $\hat{y}_{i,1} ,\dots, \hat{y}_{i,k}$.}}

\item\label{step:hatytj} If $\sigma_i=1$ then sample $\hat{y}_i\in\{0,1\}$ at random. Else let $\hat{y}_i = {\rm majority}\{ \hat{y}_{i,1} ,\dots, \hat{y}_{i,k} \}$

\item Output the bit $\hat{y}_i$ as a prediction, and obtain a ``true'' label $y_i$

\item Feed $y_i$ to $\AAA_{\ell_i}$

\item[{\small \gray{\%}}] {\small \gray{Note that $\AAA_{\ell}$ is the only copy of $\AAA$ that changes its state during this iteration}}

\end{enumerate}
\vspace{5px}
\end{enumerate}

\end{algorithm*}

We now analyze the privacy guarantees of \texttt{POP}.

\begin{theorem}\label{thm:POPprivacy}
Algorithm \texttt{POP} is $\left(O(\eps),O(\delta)\right)$-Challenge-DP. That is,
    For every adversary $\BBB$ it holds that 
    $\texttt{OnlineGame}_{\texttt{POP},\BBB}$ is 
    $\left(O(\eps),O(\delta)\right)$-DP w.r.t.\ the bit $b$ (the input of the game).
\end{theorem}

\begin{proof}
Let $\BBB$ be an adversary that playes in \texttt{OnlineGame} against \texttt{POP}, posing at most 1 challenge. That is, at one time step $i$, the adversary specifies two inputs $(x_i^0,y_i^0),(x^1_i,y^1_i)$, algorithm \texttt{POP} processes $(x_i^b,y_i^b)$, and the adversary does not see the prediction $\hat{y}_i$ at this time step. We need to show that the view of the adversary is DP w.r.t.\ the bit $b$. 
To show this, we observe that the view of $\BBB$ can be generated (up to a small statistical distance of $\delta$) by interacting with \texttt{ChallengeAT} as in the game \texttt{ChallengeAT\text{-}Game}. Formally, consider the following adversary $\hat{\BBB}$ that simulates $\BBB$ while interacting with \texttt{ChallengeAT} instead of \texttt{POP}.

\begin{algorithm*}[ht!]
\caption{\bf $\boldsymbol{\hat{\BBB}}$}

{\bf Setting:} This is an adversary that plays against \texttt{ChallengeAT} in the game \texttt{ChallengeAT\text{-}Game}.

\begin{enumerate}[topsep=-1pt,rightmargin=5pt,itemsep=-1pt]

\item Specify two datasets $S_0=\{0\}$ and $S_1=\{1\}$.

\item Instantiate algorithm $\BBB$

\item For $i = 1,2,\dots,T$

\begin{enumerate}[topsep=-3pt,rightmargin=5pt]%

\item Obtain a challenge indicator $c_i$ and inputs $x_i^0,x_i^1$ from $\BBB$ (where $x_i^0=x_i^1$ if $c_i=0$).

\item Let $\ell_i\in[k]$ be chosen uniformly at random

\item Define the query $q_i:\{0,1\}\rightarrow\R$, where $q_i(b)=f_i$ and where $f_i$ is defined as in Step~\ref{step:ft} of \texttt{POP}.\\
{\small \gray{\% 
Note that, given $b$, this can be computed from $(x_1^0,x_1^1),\dots,(x_i^0,x_i^1)$  and $\ell_1,\dots,\ell_i$ and $y_1,\dots,y_{i-1}$. Furthermore, whenever $c_i=0$ then this is a query of sensitivity at most $1$. When $c_i=1$ the sensitivity might be large, which we view it as {\em two} separate queries, corresponding to a challenge round when playing against \texttt{ChallengeAT}.}}

\item Output the challenge bit $c_i$ and the query $q_i$, which is given to \texttt{ChallengeAT}.

\item If $c_i=0$ then 

\begin{enumerate}
    
\item Obtain an outcome $\sigma_i$ from \texttt{ChallengeAT}

\item\label{step:fakehatyt} Define $\hat{y}_i$ as in Step~\ref{step:hatytj} of \texttt{POP}, as a function of $\sigma_i$ and $(x_1^0,x_1^1),\dots,(x_i^0,x_i^1)$ and $\ell_1,\dots,\ell_i$ and $y_1,\dots,y_{i-1}$.

\item Feed the bit $\hat{y}_i$ to the adversary $\BBB$ 

\end{enumerate}

\item Obtain a ``true'' label $y_i$ from the adversary $\BBB$.

\end{enumerate}
\vspace{5px}
\end{enumerate}

\end{algorithm*}

As $\hat{\BBB}$ only interacts with \texttt{ChallengeAT}, its view  at the end of the execution (which includes the view of the simulated $\BBB$) is DP w.r.t.\ the bit $b$. Furthermore, the view of the simulated $\BBB$ generated in this process is almost identical to the view of $\BBB$ had it interacted directly with \texttt{POP}. Specifically, the only possible difference is that the computation of $\hat{y}_i$ in Step~\ref{step:fakehatyt} of $\hat{\BBB}$ might not be well-defined. But this does not happen when \texttt{ChallengeAT} maintains correctness, which holds with probability at least $1-\delta$.

Overall, letting $\texttt{ChallengeAT-Game}_{\hat{\BBB}\restrict{\BBB}}$ denote the view of the simulated $\BBB$ at the end of the interaction of $\hat{\BBB}$ with \texttt{ChallengeAT}, we have that

\begin{align*}
\texttt{OnlineGame}_{\texttt{POP},\BBB}(0)
&\approx_{(0,\delta)}
\texttt{ChallengeAT-Game}_{\hat{\BBB}\restrict{\BBB}}(0)\\
&\approx_{(\eps,\delta)}
\texttt{ChallengeAT-Game}_{\hat{\BBB}\restrict{\BBB}}(1)\\
&\approx_{(0,\delta)}
\texttt{OnlineGame}_{\texttt{POP},\BBB}(1).
\end{align*}
\end{proof}

We proceed with the utility guarantees of \texttt{POP}. See Appendix~\ref{sec:agnostic} for an extension to the agnostic setting.

\begin{theorem}\label{thm:POPrealizable}
When executed with a learner $\AAA$ that makes at most $d$ mistakes and with parameters $k=\tilde{O}\left( \frac{d}{\eps^2}  \log^2(\frac{1}{\delta}) \log^2(\frac{T}{\beta}) \right)$ and $r=O\left( dk+\ln\left(\frac{1}{\beta}\right) \right)$, then with probability at least $(1-\beta)$ the number of mistakes made by algorithm $\texttt{POP}$ is bounded by 
$\tilde{O}\left( \frac{d^2}{\eps^2}  \log^2(\frac{1}{\delta}) \log^2(\frac{T}{\beta}) \right).$
\end{theorem}

\begin{proof}
By Theorem~\ref{thm:ChallengeAT}, with  probability $(1-\beta)$ over the internal coins of \texttt{ChallengeAT}, for every input sequence, its answers are accurate up to error of
$${\rm error}_{\rm CAT}=O\left( \frac{\Delta}{\eps}\sqrt{r+\lambda}\ln(\frac{r+\lambda}{\delta})\log(\frac{T}{\beta}) \right),
$$
where in our case, the sensitivity $\Delta$ is $1$, and the error of the counter $\lambda$ is at most $O\left(\frac{1}{\eps}\log(T)\log\left(\frac{T}{\delta}\right)\right)$ by Theorem~\ref{thm:counter}. We continue with the proof assuming that this event occurs. Furthermore, we set $k=\Omega\left( {\rm error}_{\rm CAT} \right)$, large enough, such that if less than $\frac{1}{5}$ the experts disagree with the other experts, then algorithm \texttt{POP} returns the majority vote with probability 1.

Consider the execution of algorithm \texttt{POP} and 
define  $1/5$-Err be a random variable that counts the number of time steps in which at least $1/5$th of the experts make an error. That is 
\begin{align*}
{\rm \mbox{1/5-Err}}  = \left|\left\{ i\in[T] : 
\sum_{j\in[k]}\1\{\hat{y}_{i,j}\neq y_i\}
 > k/5 \right\}\right|.
\end{align*}
We also define the random variable 
\begin{align*}
{\rm expertAdvance}  = \left|\left\{ i\in[T] : y_i\neq\hat{y}_{i,\ell_i} \right\}\right|.
\end{align*}
That is
{\rm expertAdvance} counts the number of times steps  in which the random expert we choose (the $\ell_i$th expert) errs.
Note that the $\ell_i$th expert is the expert that gets the ``true'' label $y_i$ as feedback. As we run $k$ experts, and as each of them is guaranteed to make at most $d$ mistakes, we get that
$$
{\rm expertAdvance}\leq kd.
$$
We now show that with high probability ${\rm \mbox{1/5-Err}}$ is not much larger than ${\rm expertAdvance}$. 
Let $i$ be a time step in which at least $1/5$ fraction of the experts  err. 
As the choice of $\ell_i$ (the expert we update) is random, then 
with probability at least $\frac{1}{5}$ the chosen expert also errs.  It is therefore unlikely that ${\rm \mbox{1/5-Err}}$ is much larger than ${\rm expertAdvance}$, which is bounded by $kd$. Specifically, by standard concentration arguments (see Appendix~\ref{sec:coinFlipping} for the precise version we use) it holds that
$$
\Pr\left[{\rm \mbox{1/5-Err}}>18dk+18+\ln\left(\frac{1}{\beta}\right)\right]\leq\beta.
$$
Note that when at least $1/5$ of the experts disagree with other experts then at least $1/5$ of the experts err. 
It follows that 
 ${\rm \mbox{1/5-Err}}$ upper bounds the number of times in which algorithm \texttt{ChallengeAT} returns an ``above threshold'' answer. Hence, by setting $r>18dk+18+\ln\left(\frac{1}{\beta}\right)$ we ensure that w.h.p.\ algorithm \texttt{ChallangeAT} does not halt prematurely (and hence \texttt{POP} does not either).

Furthermore our algorithm errs either when there is a large disagreement between the experts or when all experts err. It follows that  ${\rm \mbox{1/5-Err}}$ also upper bounds the number of times which our algorithm errs.

Overall, by setting $r=O\left( dk+\ln\left(\frac{1}{\beta}\right) \right)$  we ensure that \texttt{POP} does not halt prematurely, and by setting 
$k=O\left( \frac{\Delta}{\eps}\sqrt{r+\lambda}\ln(\frac{r+\lambda}{\delta})\log(\frac{T}{\beta}) \right)$ we ensure that \texttt{POP} does not err too many times throughout the execution.
Combining the requirement on $r$ and on $k$, it suffices to take
$$
k=\tilde{O}\left( \frac{d}{\eps^2}  \log^2(\frac{1}{\delta}) \log^2(\frac{T}{\beta}) + \frac{1}{\eps\cdot d}\log(T)\log\left(\frac{T}{\delta}\right)\right),
$$
in which case algorithm \texttt{POP} makes at most $\tilde{O}\left( \frac{d^2}{\eps^2}  \log^2(\frac{1}{\delta}) \log^2(\frac{T}{\beta}) \right)$ with high probability.
\end{proof}

\bibliographystyle{plainnat}

\appendix

\section{General Variant of challenge-DP}\label{sec:generalDef}

\begin{definition}
Consider an algorithm $\MMM$ that, in each phase $i\in[T]$, conducts an arbitrary interaction with the $i$th user. We say that algorithm $\MMM$ is {\em $(\eps,\delta)$-challenge differentially private} if for any adversary $\BBB$ we have that  $\texttt{GeneralGame}_{\MMM,\BBB,T}$, defined below, is $(\eps,\delta)$-differentially private (w.r.t.\ its input bit $b$).
\end{definition}

\begin{algorithm*}[ht!]
\caption{\bf $\boldsymbol{\texttt{GeneralGame}_{\MMM,\BBB,T}}$}\label{alg:GeneralGame}

{\bf Setting:} $T\in\N$ denotes the number of phases. $\MMM$ is an interactive algorithm and $\BBB$ is an adaptive and interactive adversary.

{\bf Input of the game:} A bit $b\in\{0,1\}$. (The bit $b$ is unknown to $\MMM$ and $\BBB$.)
\begin{enumerate}[topsep=-1pt,rightmargin=5pt,itemsep=-1pt]

\item For $i = 1,2,\dots,T$

\begin{enumerate}[topsep=-3pt,rightmargin=5pt]%

\item The adversary $\BBB$ outputs a bit $c_i\in\{0,1\}$, under the restriction that $\sum_{j=1}^i c_j\leq 1$.

\item The adversary $\BBB$ chooses two interactive algorithms $\III_{i,0}$ and $\III_{i,1}$, under the restriction that if $c_i=0$ then $\III_{i,0}=\III_{i,1}$.

\item Algorithm $\MMM$ interacts with $\III_{i,b}$. Let $\hat{y}_i$ denote the view of $\III_{i,b}$ at the end of this interaction.

\item If $c_i=0$ then set $\tilde{y}_i=\hat{y}_i$. Otherwise set $\tilde{y}_i=\bot$.

\item The adversary $\BBB$ obtains $\tilde{y}_i$.
\end{enumerate}

\item Output $\BBB$'s view of the game.
\vspace{5px}
\end{enumerate}

\end{algorithm*}

\section{A Coin Flipping Game}\label{sec:coinFlipping}

Consider algorithm~\ref{alg:CoinGame} which specifies an $m$-round ``coin flipping game'' against an adversary $\BBB$. In this game, the adaptively chooses the biases of the coins we flip. In every flip, the adversary might gain a reward or incur a ``budget loss''. The adversary aims to maximize the rewards it collects before its budget runs out.

\begin{algorithm*}[ht!]
\caption{\bf  $\boldsymbol{\texttt{CoinGame}_{\BBB,k,m}}$}\label{alg:CoinGame}

{\bf Setting:} $\BBB$ is an adversary that determins the coin biases adaptively. $k$ denotes the ``budget'' of the adversary. $m$ denotes the number of iterations.

\begin{enumerate}[topsep=-1pt,rightmargin=5pt,itemsep=-1pt]

    \item Set ${\rm budget}=k$ and ${\rm reward}=0$.

    \item In each round $i=1,2,\dots,m$:

\begin{enumerate}
	\item The adversary chooses $0\leq p_i\leq\frac{5}{6}$ and $\frac{p_i}{5}\leq q_i\leq1-p_i$, possibly based on the first $(i-1)$ rounds.
	\item A random variable $X_i\in\{0,1,2\}$ is sampled, where $\Pr[X_i=1]=p_i$ and $\Pr[X_i=2]=q_i$ and $\Pr[X_i=0]=1-p_i-q_i$.

 \item The adversary obtains $X_i$

 \item If $X_i=1$ and ${\rm budget}>0$ then ${\rm reward}={\rm reward}+1$. 

\item Else if $X_i=2$ then ${\rm budget}={\rm budget}-1$.

\end{enumerate}

\item Output ${\rm reward}$.

\vspace{5px}
\end{enumerate}

\end{algorithm*}

The next theorem states that no adversary can obtain reward much larger than $k$ in this game. Intuitively, this holds because in every time step $i$, the probability of $X_i=2$ is not much smaller than the probability that $X_i$, then (w.h.p.)\ it is very unlikely that the number of rewards would be much larger than $k$.

\begin{theorem}[\citep{GuptaLMRT10,KaplanMS21}]\label{thm:CoinGame}
For every adversary's strategy, every $k\geq0$, every $m\in\N$, and every $\lambda\in\R$, we have $$\Pr[\texttt{CoinGame}_{\BBB,k,m}>\lambda]\leq\exp\left(-\frac{\lambda}{6}+3(k+1)\right).$$
\end{theorem}

\section{Extension to the Agnostic Case}\label{sec:agnostic}

In this section we extend the analysis of \texttt{POP} to the agnostic setting. We use the tilde-notation to hide logarithmic factors in $T,\frac{1}{\delta},\frac{1}{\beta},\frac{1}{\eps}$.

\begin{theorem}[\citep{Ben-DavidPS09}]\label{thm:BenDavid}
For any hypothesis class $H$ and scalar $M^*\geq0$ there exists an online learning algorithm such that for any sequence $((x_1,y_1),\dots,(x_T,y_T))$ satisfying 
$\min\limits_{h\in H} \sum_{i=1}^T|h(x_i)-y_i|\leq M^*$ 
the predictions $\hat{y}_1,\dots,\hat{y}_T$ given by the algorithm satisfy
$$
\sum_{i=1}^T |\hat{y}_i-y_i| \leq O\left( M^* + \Ldim(H)\ln(T) \right).
$$
\end{theorem}

\begin{definition}
For parameters $u<w$, let $\texttt{POP}_{[u,w]}$ denote a variant of \texttt{POP} in which we halt the execution after the $v$th time in which we err, for some arbitrary value $u\leq v \leq w$. (Note that the execution might halt even before that, by the halting condition of \texttt{POP} itself.) This could be done while preserving privacy (for appropriate values of $u<w$) by using the counter of Theorem~\ref{thm:counter} for privately counting the number of mistakes.
\end{definition}

\begin{lemma}\label{lem:phase}
Let $H$ be a hypothesis class with $d=\Ldim(H)$, and let $\AAA$ denote the non-private algorithm from Theorem~\ref{thm:BenDavid} with $M^*=d\ln(T)$. 
Denote $k=\tilde{\Theta}\left( \frac{d^2}{\eps} \right)$, $r=u=\Theta\left(kd\ln(T)\right)$, and $w=2u$.
Consider executing $\texttt{POP}_{[u,w]}$ with $\AAA$ and with parameters $k,r$ on an adaptively chosen sequence of inputs $(x_1,y_1),\dots,(x_{i^*},y_{i^*})$, where $i^*\leq T$ denotes the time at which $\texttt{POP}_{[u,w]}$ halts. Then, with probability at least $(1-\beta)$ 
it holds that
$$\OPT_{i^*}\triangleq\min\limits_{h\in H} \sum_{i=1}^{i^*}|h(x_i)-y_i|>d\cdot\ln(T).$$
\end{lemma}

\begin{proofsketch}
Similarly to the proof of Theorem~\ref{thm:POPrealizable}, we set 
$k=\tilde{\Omega}\left( \frac{d^2}{\eps} \right)$, and assume that if less than $\frac{1}{5}$ the experts disagree with the other experts, then algorithm $\texttt{POP}_{[u,w]}$ returns the majority vote with probability 1.

Let $1/5$-Err denote the random variable that counts the number of time steps in which at least $1/5$th of the experts make an error. As in the proof of Theorem~\ref{thm:POPrealizable}, $1/5$-Err upper bounds both the number of mistakes made by $\texttt{POP}_{[u,w]}$ , which we denote by ${\rm OurError}$, as well as the number of times in which algorithm \texttt{ChallengeAT} returns an ``above threshold'' answer, which we denote by ${\rm NumTop}$. By Theorem~\ref{thm:ChallengeAT}, we know that (w.h.p.)\ ${\rm NumTop}\geq r$. 
Also let ${\rm WorstExpert}$ denote the largest number of mistakes made by a single expert. 

Consider the time $i^*$ at which $\texttt{POP}_{[u,w]}$ halts. If it halts  because $u\leq v\leq w$ mistakes have been made, then
$$
k\cdot {\rm WorstExpert} \geq 1/5\text{-Err} \geq {\rm OurError} \geq u = \Omega\left(kd\ln(T)\right).
$$
Alternatively, if $\texttt{POP}_{[u,w]}$ halts after $r$ ``above threshold'' answer, then
$$
k\cdot {\rm WorstExpert} \geq 1/5\text{-Err} \geq {\rm NumTop} \geq r = \Omega\left(kd\ln(T)\right).
$$
At any case, when $\texttt{POP}_{[u,w]}$ halts it holds that at least one expert made at least $\Omega\left(d\ln(T)\right)$ mistakes. Therefore, by Theorem~\ref{thm:BenDavid}, we have that $\OPT_{i^*}\geq d\ln(T)$.

\end{proofsketch}

\begin{theorem}
Let $H$ be a hypothesis class with $\Ldim(H)=d$.
There exists an $(\eps,\delta)$-Challenge-DP online learning algorithm providing the following guarantee. When executed on an adaptively chosen sequence of inputs $(x_1,y_1),\dots,(x_T,y_T)$, then the algorithm makes at most $\tilde{O}\left( \frac{d\cdot\OPT}{\eps^2} + \frac{d^2}{\eps^2}\right)$ mistakes (w.h.p.), where 
$$\OPT\triangleq\min\limits_{h\in H} \sum_{i=1}^{T}|h(x_i)-y_i|.$$
\end{theorem}

\begin{proofsketch}
This is obtained by repeatedly re-running $\texttt{POP}_{[u,w]}$, with the parameter setting specified in Lemma~\ref{lem:phase}. We refer to the time span of every single execution of $\texttt{POP}_{[u,w]}$ as a {\em phase}.

By construction, in every phase, $\texttt{POP}_{[u,w]}$ makes at most $w=\tilde{\Theta}(kd)$ mistakes. By Lemma~\ref{lem:phase} {\em every} hypothesis in $H$ makes at least $d\cdot\ln(T)$ mistakes in this phase. Therefore, there could be at most $\tilde{O}\left( \max\left\{1 \,,\,\frac{\OPT}{d} \right\}\right)$ phases, during which we incur a total of at most $\tilde{O}\left( \frac{d\cdot\OPT}{\eps^2} + \frac{d^2}{\eps^2}\right)$ mistakes.
\end{proofsketch}

\end{document}